\documentclass[10pt,twocolumn,letterpaper]{article}

\usepackage[pagenumbers]{iccv} 

\newcommand{\red}[1]{{\color{red}#1}}

\usepackage{xcolor}
\definecolor{Gray}{gray}{0.9}
\definecolor{mygreen}{rgb}{0.0, 0.5, 0.0}
\definecolor{myred}{rgb}{0.8, 0.25, 0.33}
\definecolor{myblue}{rgb}{0.19, 0.55, 0.91}
\definecolor{uclablue}{rgb}{0.15, 0.45, 0.68}
\definecolor{boxgreen}{rgb}{0.02, 0.66, 0.02}
\definecolor{boxred}{rgb}{0.66, 0.1, 0.1}
\definecolor{boxblue}{rgb}{0.01, 0.01, 0.73}
\definecolor{mygray}{gray}{0.4}
\usepackage{notation}
\usepackage{times}
\usepackage{cuted}
\usepackage{arydshln}
\usepackage{marvosym}
\usepackage{graphicx}
\usepackage{amssymb}
\usepackage{url}
\usepackage{caption}

\usepackage{microtype}
\usepackage{booktabs}
\usepackage{threeparttable}

\usepackage{amsmath}
\usepackage{mathtools}
\usepackage{amsthm}

\usepackage{soul}

\usepackage{listings}
\usepackage{etoc}
\usepackage{algorithm}
\usepackage{algorithmic}
\usepackage{lipsum}
\usepackage{pifont}
\usepackage{wrapfig}
\usepackage{colortbl}
\usepackage{acronym}
\usepackage{multicol}
\usepackage{multirow}
\usepackage{makecell}
\usepackage{enumitem}
\usepackage{tabularx}
\usepackage{colortbl}
\usepackage{array}
\usepackage[skip=3pt,font=small]{subcaption}
\usepackage[skip=3pt,font=small]{caption}
\usepackage{xspace}
\usepackage{mdframed}

\usepackage[export]{adjustbox}

\newcolumntype{Y}{>{\arraybackslash}X}

\usepackage[utf8]{inputenc} 
\usepackage[T1]{fontenc}    
\usepackage{amsfonts}       
\usepackage{nicefrac}       

\usepackage{mathtools}
\usepackage{amssymb}
\usepackage{amsthm}
\usepackage{autobreak}

\renewcommand{\paragraph}[1]{\noindent\textbf{#1.}}

\renewcommand{\paragraph}[1]{\noindent\textbf{#1.}}

\makeatletter
\DeclareRobustCommand\onedot{\futurelet\@let@token\@onedot}
\def\@onedot{\ifx\@let@token.\else.\null\fi\xspace}

\def\ie{\emph{i.e}\onedot}

\def\etc{\emph{etc}\onedot}

\makeatother

\acrodef{llms}[LLMs]{Large Language Models}
\acrodef{mlms}[MLMs]{Multimodal Language Models}

\usepackage{xcolor}

\renewcommand{\red}{\textcolor{Red}}
\newcommand{\green}{\textcolor{Green}}

\usepackage[accsupp]{axessibility}  

\definecolor{iccvblue}{rgb}{0.21,0.49,0.74}
\usepackage[pagebackref,breaklinks,colorlinks,allcolors=iccvblue]{hyperref}

\def\paperID{12617} 
\def\confName{ICCV}
\def\confYear{2025}

\title{Open-World Skill Discovery from Unsegmented Demonstration Videos}

\author{
Jingwen Deng$^{1}$\thanks{Equal Contribution.\quad$^{\dagger}$Corresponding Author.}
\and 
Zihao Wang$^{1}$\footnotemark[1]
\and 
Shaofei Cai$^1$ 
\and
Anji Liu$^2$ 
\and
Yitao Liang$^{1\dagger}$
\and
\small{$^1$Peking University~\,
$^2$University~of~California,~Los Angeles~\,}
\\
\tt{\small{\{dengjingwen,zhwang,caishaofei\}@stu.pku.edu.cn, liuanji@cs.ucla.edu, yitaol@pku.edu.cn}}
}

\begin{document}
\maketitle
\renewcommand{\thefootnote}{}
\footnotetext{\textit{Preprint version of the paper accepted at ICCV 2025.}}
\renewcommand{\thefootnote}{\arabic{footnote}}
\begin{abstract}
    Learning skills in open-world environments is essential for developing agents capable of handling a variety of tasks by combining basic skills.
Online demonstration videos are typically long but unsegmented, making them difficult to segment and label with skill identifiers.
Unlike existing methods that rely on random splitting or human labeling, we have developed a self-supervised learning-based approach to segment these long videos into a series of semantic-aware and skill-consistent segments.
Drawing inspiration from human cognitive event segmentation theory, we introduce \textbf{Skill Boundary Detection} (SBD), an annotation-free temporal video segmentation algorithm. SBD detects skill boundaries in a video by leveraging prediction errors from a pretrained unconditional action-prediction model. This approach is based on the assumption that a significant increase in prediction error indicates a shift in the skill being executed. 
We evaluated our method in Minecraft, a rich open-world simulator with extensive gameplay videos available online. The SBD-generated segments yielded relative performance improvements of 63.7\% and 52.1\% for conditioned policies on short-term atomic tasks, and 11.3\% and 20.8\% for their corresponding hierarchical agents on long-horizon tasks, compared to random segmented baselines.
Our method can leverage the diverse YouTube videos to train instruction-following agents.
The project page is at \href{https://craftjarvis.github.io/SkillDiscovery/}{https://craftjarvis.github.io/SkillDiscovery/}.
\end{abstract}

\newcommand{\omnijarvisresult}{11.3\%}

\section{Introduction}

\begin{table}[!ht]
\renewcommand{\arraystretch}{1.2}
\hspace{0.2in}
\centering
\small
\setlength{\tabcolsep}{0.5pt}
\begin{tabularx}{0.99\linewidth}{@{}>{\centering\arraybackslash}m{0.6in} >{\centering\arraybackslash}m{1.4in} >{\centering\arraybackslash}m{0.99\linewidth-2.0in}@{}}
\toprule
\textbf{Method} & \textbf{Illustration} & \textbf{Type} \\ \midrule
Random splitting & \includegraphics[width=1.1in]{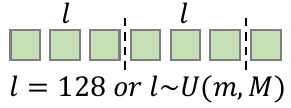}  &\footnotesize \makecell[b]{rule-based} \\ \midrule
Reward-driven & \includegraphics[width=1.3in]{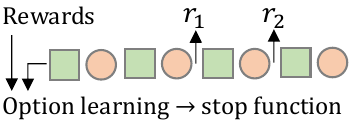}  &\footnotesize \makecell[b]{rule-based} \\ \midrule
Top-down & \includegraphics[width=1.4in]{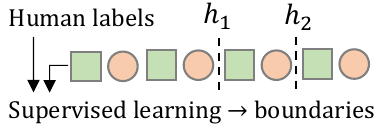}  &\footnotesize \makecell[b]{rule-based} \\ \midrule
Bottom-up & \includegraphics[width=1in]{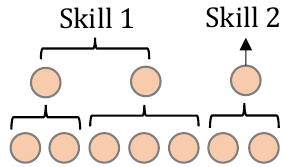}  &\footnotesize rule-based \\ \midrule \midrule
\textbf{Ours (SBD)} & \includegraphics[width=1in]{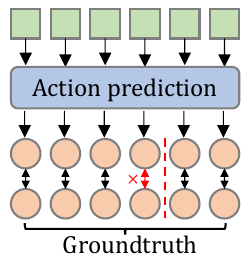} & learning-based \\
\bottomrule
\end{tabularx}
\caption{\textbf{Comparisons between existing segmentation methods and our method SBD.} 
Existing methods usually rely on human-designed rules, while our method is learning-based. 
\textbf{Random splitting} can result in a single skill spanning different segments or multiple skills located within one segment.
\textbf{Reward-driven} methods require additional reward information, which is challenging for human annotators to label.
\textbf{Top-down} methods often result in limited skill diversity and high computation costs.
\textbf{Bottom-up} methods struggle in visually partially observable environments.
\raisebox{0.15\baselineskip}{\includegraphics[valign=c,height=0.5\baselineskip]{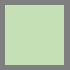}} are visual observations and \raisebox{0.15\baselineskip}{\includegraphics[valign=c,height=0.5\baselineskip]{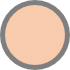}} are actions.}
\label{table: limitations}
\end{table}

Most existing LLM-based instruction-following agents adopt a two-tier architecture of a planner and a controller. The planner decomposes high-level instructions into atomic skills, which the controller then maps to actions based on current observations for interaction with the environment~\citep{omnijarvis,deps,jarvis-1,li2024optimus,voyager,cheng2024exploring}. Training such agents typically involves segmenting long trajectories into atomic skill sequences, followed by separate training of the planner and controller. 

Learning skills from long-sequence videos is critical for building such hierarchical agents. 
However, real-world video data are often long-horizon, unsegmented, and with no fine-grained skill labels.
In addition, the concept of a ``skill'' is ill-defined and varies widely across domains such as video gaming~\citep{groot1,voyager,mastering,proagent}, robotic control~\citep{rt-2}, and autonomous driving~\citep{auto-driving}. This ambiguity, along with the immense diversity of skills in open worlds, makes skill learning particularly challenging, especially in partially observable settings. To enable open-world skill learning from unsegmented demonstration videos, the first challenge is to segment them into semantically meaningful, self-contained skills. 

We list the existing segmentation methods in \cref{table: limitations}. The naive \textbf{random splitting} methods~\citep{groot1,steve1} divide videos into segments of predefined lengths (e.g., fixed length, uniform distribution). However, they do not ensure that each segment contains a distinct skill. Additionally, predefined lengths may not match the actual distribution of skill length in real world (see \cref{subsec:visual}). \textbf{Reward-driven} methods~\citep{option} discover skills through the environment's reward signal. They cannot capture skills with no associated rewards and may split one skill into multiple segments when rewards are repeatedly gained during execution. \textbf{Top-down} methods~\citep{taco,rt-h} rely on predefined skill sets from human experts. They use manual labeling or supervised learning to segment videos. Although this method can produce reasonable results, it is expensive and limited by the narrow range of predefined skills. \textbf{Bottom-up} methods~\citep{buds,FAST} use algorithms such as agglomerative clustering or byte-pair encoding~\citep{bpe} to split action sequences. However, they struggle in visually partially observable settings where both observations and actions must be considered. All the methods above rely on human-designed rules to segment the videos, which highlights the need for a learning-based method that can adaptively segment skills from unsegmented videos in open world.

Inspired by event segmentation theory (EST)~\citep{EST}, which proposes that humans naturally split continuous experiences into discrete events when prediction errors in perceptual expectations rise, we introduce Skill Boundary Detection (SBD)—a method for automatically identifying potential skill boundaries in long trajectories. 
SBD employs a predictive model trained on a dataset of unsegmented videos to predict future actions based on past observations, effectively capturing temporal dependencies~\citep{vpt}. We then use this pretrained model to make predictions on unsegmented videos and compare them to the ground-truth actions. A significant increase in prediction error indicates a shift in the skill being executed (\cref{thm:non-switching and switching}), enabling us to detect boundaries between different skills. SBD relies on self-supervised learning, removing the need for additional human labeling. This allows SBD to utilize a wide range of YouTube videos to train instruction-following agents.

We evaluate SBD in Minecraft~\citep{minerl,minedojo,mcu}, a rich open world simulator. First, we apply SBD to create a segmented Minecraft video dataset. We then re-train a video-conditioned policy~\citep{groot1} and a language-conditioned policy~\citep{steve1} on this dataset. We evaluate them on a diverse Minecraft skills benchmark~\citep{mcu}. The improved versions of these policies result in relative performance increases of 63.7\% and 52.1\% compared to the original ones, respectively. We also test their corresponding hierarchical agents that use behavioral cloning and in-context learning, achieving relative performance increases of {\omnijarvisresult} and 20.8\%, respectively. These findings show the effectiveness of our proposed method for skill discovery from unsegmented demonstrations and its potential to advance open-world skill learning.

\begin{figure*}[!ht]
\begin{center}
\includegraphics[width=0.98\linewidth]{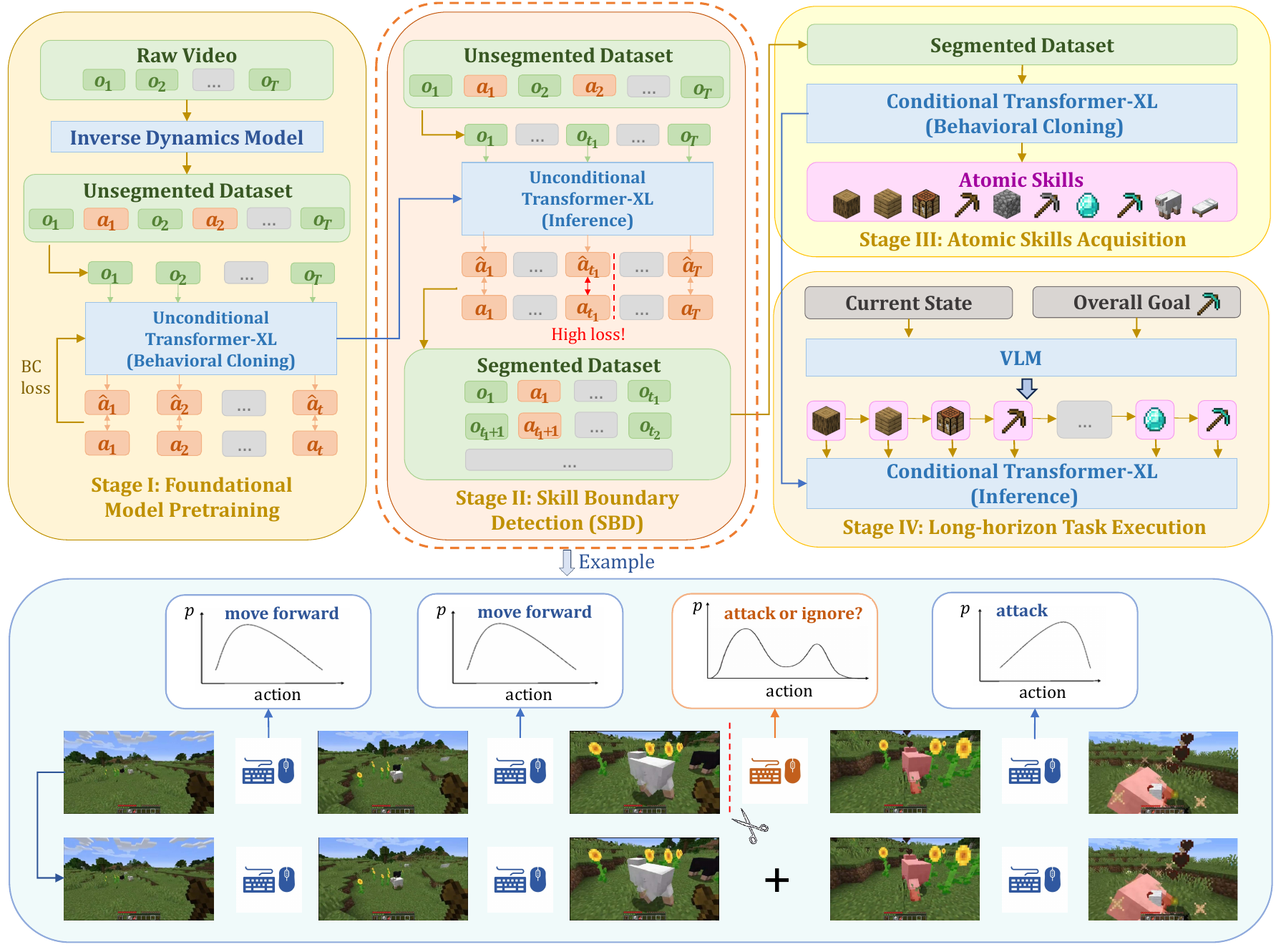}
\end{center}
\caption{
\textbf{Pipeline of our method SBD for discovering skills from unsegmented demonstration videos.} 
\textbf{Stage I:} An \textit{unconditional} Transformer-XL based policy model~\citep{transformerxl,vpt} is pretrained on an unsegmented dataset to predict future actions (labeled by an inverse dynamics model) based on past observations using behavioral cloning.
\textbf{Stage II:} The pretrained unconditional policy will produce a high predicted action loss when encountering uncertain observations (e.g., deciding whether to kill a new sheep) in open worlds. These timesteps should be marked as skill boundaries, indicating the need for additional instructions to control behaviors. We segment the long unsegmented videos into a series of short atomic skill demonstrations.
\textbf{Stage III:} We train a \textit{conditional} Transformer-XL based policy model on the segmented dataset to master a variety of atomic skills.
\textbf{Stage IV:} Finally, we use hierarchical methods (a combination of vision-language models and the conditional policy) to model the long demonstration videos and follow long-horizon instructions.
}
\label{fig:pipeline}
\end{figure*}

\section{Problem Formulation}
\label{sec:preliminaries}

\paragraph{Skills in Behavioral Cloning}
Behavioral cloning learns a policy $\pi: \mathcal{O} \rightarrow \mathcal{A}$ from expert demonstrations, mapping observations to actions based on pairs $(o_i, a_i)_{i=1}^T$. In partially observable settings, the policy typically takes the observation history $o_{1:t}$ as input. To enable goal-directed behavior, recent works~\citep{vln-clip,ogn1,traj} learn goal-conditioned policies $\pi(o_{1:t}, g)$, where $g$ denotes a target goal. Such goal-conditioned policies are referred to as skills, as they encode reusable behaviors for achieving specific goals.

\paragraph{Hierarchical Agent Architecture}
Directly learning policy $\pi(o_{1:t}, g)$ is challenging in complex open-world environments like Minecraft. Goals such as \texttt{build a house} or \texttt{mine diamond blocks} involve numerous intermediate steps, making it computationally and conceptually difficult to model an agent's behavior without breaking the goal into manageable parts. To tackle this, hierarchical agent architectures are often used~\citep{voyager,palm-e,rocket-1,alma}. These architectures consist of a high-level planner and a low-level controller. The planner decomposes the overall goal into a series of sub-goals and decides which one the controller should address. The controller then focuses on achieving the current sub-goal by employing a specific skill.
Specifically, for a certain $t$, the policy can be represented as: 
\begin{equation}
    \pi(a|o_{1:t}, g) = \pi_1(g_t|o_{1:t}, g)\pi_2({a|o_{1:t}, g_t})
\end{equation} where $\pi_1$ is the planner's policy and $\pi_2$ is the controller's policy. The planner periodically (not at every step) checks whether the current sub-goal has been finished and decides whether to change it, so $g_t$ does not change frequently. 

\paragraph{Identifying Implicit Skills in Trajectories}
A critical challenge in training hierarchical agents, as described above, is the identification of skills within the trajectory data. To be more specific, given $\tau=(o_i, a_i)_{i=1}^T$, we want to split it into $\tau_1=(o_i, a_i)_{i=1}^{t_1}, \tau_2=(o_i, a_i)_{i=t_1}^{t_2}, \ldots, \tau_k=(o_i, a_i)_{i=t_{k-1}}^{T}$ where each sub-trajectories $\tau_j$ shows an independent skill $\pi(g_j)$. Since some controllers use self-supervised methods to learn sub-goals~\citep{steve1,groot1}, we don't have to identify $g_j$, but must identify $t_j$ (\ie, the moments at which the agent adopts a different skill). 
For clarity, in the subsequent sections, we denote the skill adopted at a certain time step $t$ as $\pi_t$, corresponding to $\pi(g_t)$.

\section{Method}
\label{sec:method}

We first introduce our overall pipeline to build policies and long-horizon hierarchical agents based on unsegmented videos in Section \ref{sec:pipeline}. Then we introduce how we split the unsegmented videos into skill sequences in Section \ref{sec:sbd_1}. We finally introduce the implementation details in Section \ref{sec:details}.


\subsection{Pipeline}\label{sec:pipeline}

Given a long sequence of trajectories $D = {(o_i, a_i)}_{i=1}^T$, we first train an unconditional policy $\pi_{\text{unconditional}}$ using behavioral cloning to predict the actions purely based on past observations as follows:
\begin{equation}
    \min_{\theta} \sum_{t \in [1 \dots T]} -\log \pi_{\text{unconditional}} (a_t | o_{1:t}).
\end{equation}
When training on action label-free videos, the action labels are generated from an Inverse Dynamics Model~\citep{vpt}~.
We then evaluate this unconditional policy $\pi_\text{uncondtional}$ on the entire dataset $D$ of long sequences. Our proposed method SBD segments the complete video into a series of short segments $D_{seg} = {(o_i,a_i)}_{i=m}^{n}$ based on the evaluation results, where $m$ and $n$ is the selected segmentation timestamps. $D_{seg}$ is then used for skill learning to obtain conditional policies~\citep{groot1,steve1} that take additional videos or text as instructions. 
Finally, we integrate these skills with vision language models to build hierarchical agents, enabling them to complete long-horizon tasks~\citep{omnijarvis,jarvis-1}. See the overall pipeline in \cref{fig:pipeline}.

\subsection{Skill Boundary Detection (SBD)}\label{sec:sbd_1}

\begin{algorithm}[t]
\caption{Skill Boundary Detection}\label{alg:SBD}
\begin{algorithmic}[1]
\STATE \textbf{Input:} \((o_i, a_i, e_i)_{i=1}^T\), a model \(M\) to predict \(a_t\) given \(o_{1:t}\). \(e_i\) is the boolean external information indicating whether this step should be a boundary.
\STATE \textbf{Initialize:} $\text{begin} \gets 1$, $\text{loss\_history} \gets []$, $\text{boundaries} \gets []$
\FOR{$t \gets 1$ to $T$}
    \STATE $\text{loss} \gets M(a_t \mid o_{\text{begin}:t})$
    \STATE $\text{loss\_history.append}(\text{loss})$
    \IF{$\text{loss} - \text{mean}(\text{loss\_history}) > \text{GAP}$ \textbf{or} $e_t$ \text{is true}}
        \STATE $\text{boundaries.append}(t)$
        \STATE $\text{begin} \gets t$
        \STATE $\text{loss\_history} \gets []$
    \ENDIF
\ENDFOR
\STATE \textbf{Return:} $\text{boundaries}$
\end{algorithmic}
\end{algorithm}

SBD takes a long unsegmented trajectory and the unconditional action-prediction model as input. A sliding window is used to simulate the model's memory of past observations. At each time step $t$, we use the model to predict the next action and compare it with the ground truth to compute the loss. A skill change is considered likely if the loss exceeds the average loss by a hyperparameter, GAP, or if an external indicator is true. In such cases, $t$ is marked as a boundary. The model's memory is then cleared, and the algorithm proceeds to analyze the next sub-trajectory.

We then explain the two core components of the algorithm: loss and external information.

\subsubsection{Boundary with Entropy Loss}\label{subsec:loss}

In this section, we explain the core idea of our method: why loss can indicate skill boundaries. Intuitively, without goal information, similar observations in an open world can inherently lead to different actions—even with perfect modeling. For example, upon seeing a sheep, both \textit{ignore} and \textit{attack} may be plausible depending on the (unknown) goal. Skill transitions are moments where such ambiguity is highest. Since continuing a skill is more likely than changing it, the policy learns to favor the former, making skill transitions statistically surprising. According to Bayes' rule, when the loss of the policy increases largely, it is likely that a skill transition has occurred.

To support this, we introduce three key assumptions about skills: skill consistency, skill confidence, and action deviance at skill transition. 

The first assumption is that the skill being used does not change frequently (less than $1/K$). The idea is that the agent should consistently stick to a skill unless there is a strong reason to change to another, which would result in an increase in predictive loss.

\begin{assumption}[Skill Consistency] \label{assumption: consistency}
There exists an adequately large $K$, $\forall t,$
\begin{equation}
    P(\pi_{t+1} \ne \pi_t|o_{1:t+1}) < 1/K
\end{equation}
\end{assumption}

The second assumption states that, at any given time, the agent is confident in the action it takes. In other words, the probability of the agent choosing an action with confidence $c$ is greater than $1-\delta$. This assumption ensures that the policy reliably makes decisions based on its learned skills, rather than being uncertain in its actions. This helps the unconditional model make accurate predictions when the agent does not change its skill.

\begin{assumption}[Skill Confidence] \label{assumption: confidence}
There exists $c$ and an adequately small $\delta, \forall t,$
\begin{equation}
    P(\pi_t(a_t|o_{1:t})>c) > 1 - \delta
\end{equation}
\end{assumption}

The third assumption posits that when an agent changes its skill, it is likely to perform an action that is significantly less probable under the previous skill. This "surprising" action generates a clear signal, enabling us to detect the skill transition. Although the agent might change policies without performing a surprising action, such instances suggest that the agent is in a transitional, ambiguous phase between two policies, which is inherently challenging to identify. In essence, we aim to recognize the first clear skill boundary.

\begin{assumption}[Action Deviance at Skill Transition] \label{assumption: deviance}
There exists $m$, $\forall t$, when $\pi_1=\pi_2=...=\pi_t \ne \pi_{t+1}$,
\begin{equation}
\frac{\pi_t(a_{t+1}|o_{1:t+1})}{(\prod_{i=1}^t\pi_t(a_{i}|o_{1:i}))^{1/t}} < \frac{1}{2}m
\end{equation}
\end{assumption}

By the law of total probability, we have
\begin{equation}
\begin{split}
\label{eq:transition}
    &P(a_{t+1}|o_{1:t+1}) =P(\pi_{t+1}=\pi_t|o_{1:t+1})\pi_t(a_{t+1}|o_{1:t+1}) \\     &\quad+\sum_{\pi\ne\pi_t}P(\pi_{t+1}=\pi|o_{1:t+1})\pi(a_{t+1}|o_{1:t+1})
\end{split}
\end{equation}

Intuitively, if the agent changes its skill, the predictive probability will be low. This is because, in \cref{eq:transition}, in the first term the probability of $a_{t+1}$ under the original skill $\pi_t(a_{t+1}|o_{1:t+1})$ is low, and in the second term the probability of skill transition is low. On the contrary, if the agent does not change its skill, both terms will be high, so the predictive probability will be high. Therefore, we have the following theorem regarding the bounds of relative predictive probability under scenarios of skill transition and non-transition.

\begin{theorem}[bounds of relative predictive probability] \label{thm:non-switching and switching}
If $\pi_1=\pi_2=...=\pi_{t}=\pi_{t+1}$, then we have 
\begin{equation}
P(\frac{P(a_{t+1}|o_{1:t+1})}{(\prod_{i=1}^t P(a_{i}|o_{1:i}))^{1/t}}>\frac{(K-1)c}{K}) > 1-\delta
\end{equation}

If $\pi_1=\pi_2=...=\pi_{t}\ne\pi_{t+1}$, then we have 
\begin{equation}
\begin{split}
P(\frac{P(a_{t+1}|o_{1:t+1})}{(\prod_{i=1}^t P(a_{i}|o_{1:i}))^{1/t}} < \frac{Km}{2(K-1)} + \frac{1}{c(K-1)}) \\
> 1-t\delta
\end{split}
\end{equation}
\end{theorem}
\begin{proof}
    See \cref{app: proof} for the detailed proof.
\end{proof}

If $c>m$ and $(K-4)c^2>2$, then the lower bound of relative predictive probability under skill transition in \cref{thm:non-switching and switching} is greater than the upper bound of it under skill non-transition. Therefore, the theorem demonstrates that the model predicts actions less accurately when the agent changes its skill. Our action-prediction model uses the negative log-likelihood of actions $-\log P(a_t|o_{1:t})$ as loss. Consequently, at line 6 of \cref{alg:SBD}, when the loss exceeds the average loss by a predefined threshold, it is likely that a skill transition has occurred.

\subsubsection{Boundary with External Information}\label{subsec:info}

In \cref{subsec:loss}, we demonstrate that SBD can identify the first distinguishable skill boundary using only observations and actions. However, some datasets include additional external information such as privileged in-game data or VLM annotations~\citep{guo2025seed1,comanici2025gemini}. Incorporating these auxiliary signals improves the detection of skill boundaries that are otherwise difficult to recognize through loss-based detection alone. For example, in Minecraft, crafting an item does not involve abrupt changes in mouse movement, making it hard to detect boundaries purely with predictive action loss. In such cases, in-game logs of successful crafting events provide valuable hints. It is important to note that this is an optional component of SBD; as we will discuss in \cref{subsec:ablation}, it also performs well on datasets without external auxiliary information.

\begin{table*}[!ht]
\centering
\begin{threeparttable}
\renewcommand{\arraystretch}{1.2}
\setlength{\tabcolsep}{3pt}
\resizebox{0.99\linewidth}{!}{
\begin{tabular}{@{}ccccccccccccccc@{}}
\toprule
\textbf{Task} &  & \makecell[c]{\includegraphics[height=0.6cm,valign=c,keepaspectratio]{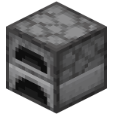} \\ use \\ furnace}
& \makecell[c]{\includegraphics[height=0.6cm,valign=c,keepaspectratio]{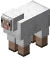} \\ hunt \\ sheep}
& \makecell[c]{\includegraphics[height=0.6cm,valign=c,keepaspectratio]{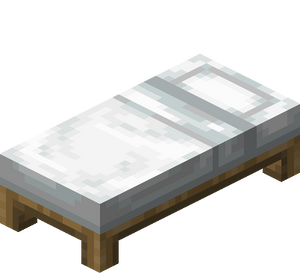} \\ sleep \\ in bed}
& \makecell[c]{\includegraphics[height=0.6cm,valign=c,keepaspectratio]{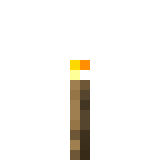} \\ use \\ torch}
& \makecell[c]{\includegraphics[height=0.6cm,valign=c,keepaspectratio]{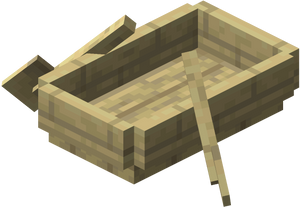} \\ use \\ boat}
& \makecell[c]{\includegraphics[height=0.6cm,valign=c,keepaspectratio]{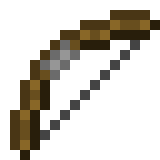} \\ use \\ bow}
& \makecell[c]{\includegraphics[height=0.6cm,valign=c,keepaspectratio]{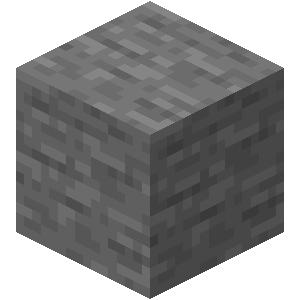} \\ collect \\ stone}
& \makecell[c]{\includegraphics[height=0.6cm,valign=c,keepaspectratio]{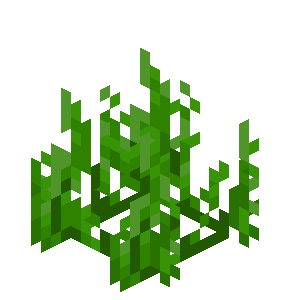} \\ collect \\ seagrass}
& \makecell[c]{\includegraphics[height=0.6cm,valign=c,keepaspectratio]{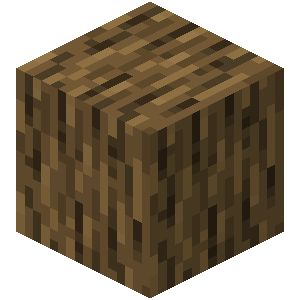} \\ collect \\ wood}
& \makecell[c]{\includegraphics[height=0.6cm,valign=c,keepaspectratio]{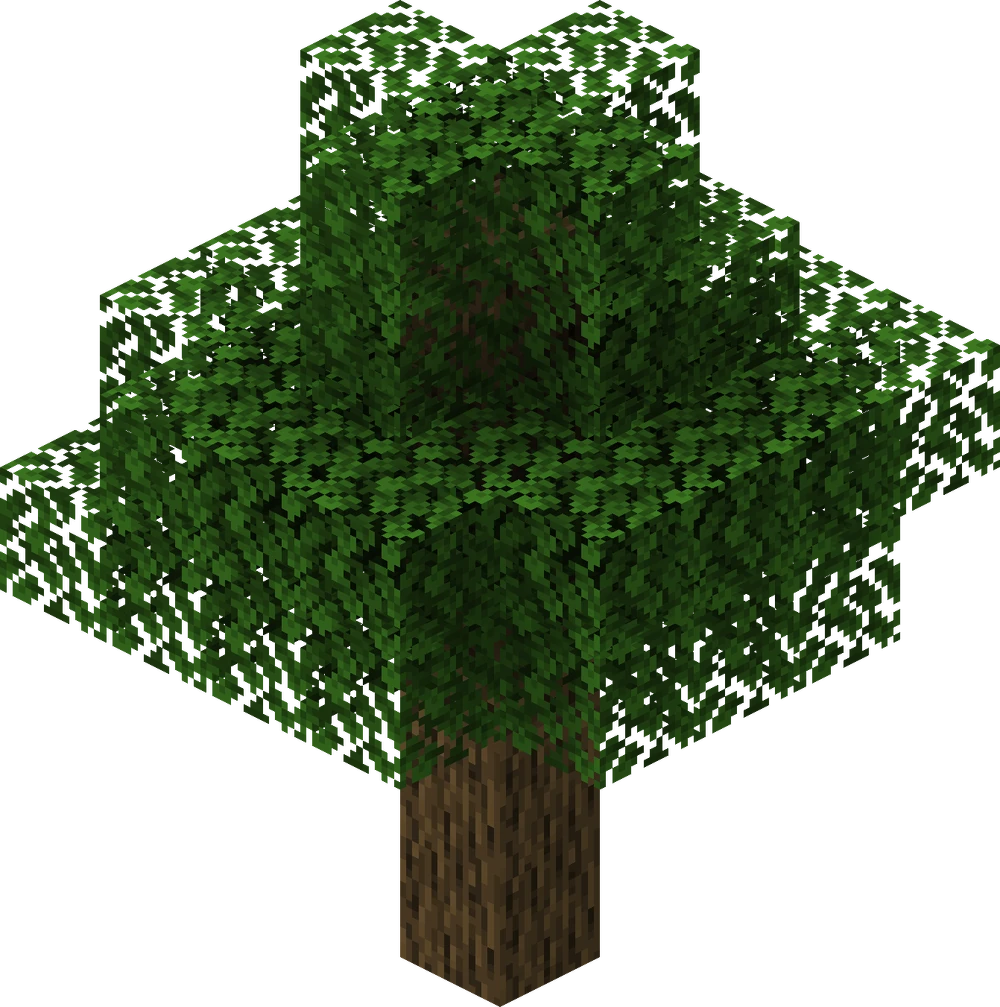} \\ find and \\ collect wood}
& \makecell[c]{\includegraphics[height=0.6cm,valign=c,keepaspectratio]{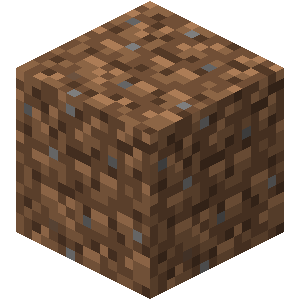} \\ collect \\ dirt}
& \makecell[c]{\includegraphics[height=0.6cm,valign=c,keepaspectratio]{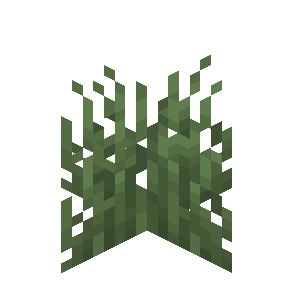} \\ collect \\ grass}
& \textbf{Average}\\  
\midrule
\multirow{1}{*}{\textbf{VPT}}
& \textit{original} & 11.0\% & 29.0\% & 0.0\% & 22.0\% & 7.0\% & 23.0\% & 77.0\% & 14.0\% & 73.0\% & 33.0\% & 30.0\% & 38.0\% & - \\
\midrule
\multirow{3}{*}{\makecell[c]{\textbf{GROOT}\\(video\\conditioned)}} 
    & \textit{original} & 21.0\% & 26.0\% & 100.0\% & 97.0\% & 71.0\% & 30.0\% & 21.7 & 34.0\% & 76.0\% & 19.0\% & 14.5 & 9.5 & -  \\  
    & \textit{\textbf{ours}} & 30.0\% & 54.0\% & 100.0\% & 88.0\% & 93.0\% & 80.0\% & 26.8 & 51.0\% & 90.0\% & 44.0\% & 19.7 & 25.4 & - \\  
    & $\Delta$ & \green{+42.9\%} & \green{+107.7\%} & 0 & \red{-9.3\%} & \green{+30.1\%} & \green{+166.7\%} & \green{+23.3\%} & \green{+50.0\%} & \green{+18.4\%} & \green{+131.6\%} & \green{+36.1\%} & \green{+166.3\%} & \textbf{+63.7\%} \\  
\midrule
\multirow{3}{*}{\makecell[c]{\textbf{STEVE-1}\\(image \& text\\conditioned}}
    & \textit{original} & 0.0\% & 1.0\% & 33.3\% & 33.3\% & 18.8\% & 65.6\% & 96.9\% & 18.8\% & 80.2\% & 57.3\% & 44.8\% & 46.9\% & - \\  
    & \textit{\textbf{ours}} & 0.0\% & 3.1\% & 40.6\% & 77.1\% & 42.7\% & 71.9\% & 96.9\% & 47.9\% & 84.4\% & 67.7\% & 43.8\% & 71.9\% & -\\  
    & $\Delta$ & 0 & $-$ & \green{+21.9\%} & \green{+131.3\%} & \green{+127.8\%} & \green{+9.5\%} & 0 & \green{+155.6\%} & \green{+5.2\%} & \green{+18.2\%} & \red{-2.3\%} & \green{+53.3\%} & \textbf{+52.1\%} \\ 
\bottomrule
\end{tabular}
}
\end{threeparttable}
\caption{\textbf{Success rate of different policies on Minecraft skill benchmark.} 
For VPT~\citep{vpt}, we report the results of the behavioral cloning version. For GROOT~\citep{groot1} and STEVE-1~\citep{steve1}, we report the results of the original and our re-trained models with SBD, respectively.
A value with \% indicates the average success rate, while a value without \% indicates the average rewards.
See more details in \cref{app: benchmark}.}
\label{main_result_short}
\end{table*}

\begin{table*}[!ht]
\centering
\small
\renewcommand{\arraystretch}{1.2}
\captionsetup[subtable]{justification=centering}

\begin{subtable}{0.37\linewidth}
\resizebox{0.98\linewidth}{!}{
    \setlength{\tabcolsep}{3pt}
    \begin{tabular}{@{}ccccccc@{}}
    \toprule
     & \textbf{Method} & Wood & Food & Stone & Iron & \textbf{Average} \\  
    \midrule
        & \textit{original} & 95\% & 44\% & 82\% & 32\%  & -  \\  
        & \textit{\textbf{ours}} & 96\% & 55\% & 90\% & 35\% & - \\  
        & $\Delta$ & \green{+1.1\%} & \green{+25.0\%} & \green{+9.8\%} & \green{+9.4\%} & \textbf{+11.3\%} \\
    \bottomrule
    \end{tabular}
}
    \caption{OmniJARVIS (behavioral cloning)}
\end{subtable}
\begin{subtable}{0.61\linewidth}
\resizebox{0.98\linewidth}{!}{
    \setlength{\tabcolsep}{3pt}
    \begin{tabular}{@{}ccccccccccccccc@{}}
    \toprule
     & \textbf{Method} & Wood & Oak & Birch & Stone & Iron & Diamond & Armor & Food & \textbf{Average} \\  
    \midrule
        & \textit{original} & 92\% & 89\% & 90\% & 90\% & 33\% & 8\% & 12\% & 39\% & -  \\  
        & \textit{\textbf{ours}} & 97\% & 95\% & 94\% & 91\% & 35\% & 10\% & 19\% & 62\% & - \\  
        & $\Delta$ & \green{+5.4\%} & \green{+6.7\%} & \green{+4.4\%} & \green{+1.1\%} & \green{+6.1\%} & \green{+25.0\%} & \green{+58.3\%} & \green{+59.0\%} & \textbf{+20.8\%} \\
    \bottomrule
    \end{tabular}
}
    \caption{JARVIS-1 (in-context learning)}
\end{subtable}
\caption{\textbf{Success rate of two agents on long programmatic tasks.} The goal-conditioned policies of the agents are trained on the dataset segmented by SBD and the original random splitting method. In each group, the agent is required to obtain a certain type of item from scratch or be given an iron pickaxe. For example, the diamond group includes diamond pickaxe, diamond sword, jukebox, \etc.}
\label{main_result_long}
\end{table*}

\subsection{Implemetation Details}\label{sec:details}

We take the Minecraft environment~\citep{minerl,zheng2023towards} as the evaluation simulator. It offers a highly complex (> $10^{10^{20}}$ states), diverse, and dynamic environment ideal for testing open-ended agents requiring \textit{OOD generalization}~\citep{mcu}. See \cref{app:minecraft_env} for more information on this environment.

We now introduce the architectures and training/inference details of the policies and agents used in \cref{sec:pipeline}.

\paragraph{Unconditional Policy and Segmentation Details}
For the unconditional policy $\pi_\text{unconditional}$ in Stage I, we use the pre-trained vpt-3x models~\citep{vpt}. 
This is a foundational Transformer-XL~\citep{transformerxl} based model pretrained on large-scale YouTube data and finetuned on the early-game dataset using behavioral cloning.
The hyperparameter GAP in Stage II is set to 18, considering the average trajectory length and semantics. As for the external information, we use event-based information such as \texttt{mine\_block:oak\_log} or \texttt{use\_item:torch} available in the contractor dataset~\citep{vpt}. To identify potential skill transitions, only the final event in a series of identical consecutive events is marked as positive, as it signifies a possible skill transition. For instance, if the agent chops a tree repeatedly and then crafts a table, the video should be split at the moment the agent mines the last block of wood. Further details are provided in \cref{app: event}.

\paragraph{Policies for Atomic Skills}
We employ both video-conditioned and language-conditioned policies to learn atomic skills from the segmented demonstration videos, as shown in Stage III.
The video-conditioned policy is built upon GROOT~\citep{groot1}, which utilizes self-supervised learning to model instructions and behaviors. In the original GROOT, a fixed-length sequence of 128 frames serves as the instruction, which is encoded into a goal embedding using a conditioned variational autoencoder (C-VAE)~\citep{vae,cvae}. The decoder then predicts actions based on the instruction and environmental observations in an auto-regressive manner.
The language-conditioned policy is derived from STEVE-1~\citep{steve1}, an instruction-tuned VPT model capable of following open-ended texts or 16-frame visual instructions. This model is trained by first adapting the pretrained VPT model to follow commands in MineCLIP's~\citep{minedojo} latent space, and then training a C-VAE model to predict latent codes from the text. 

The original GROOT and STEVE-1 policies divide the training trajectories into segments of 128 frames and uniformly sample between 15 and 200 frames, respectively. We re-train these models using our SBD method on the segmented trajectories and compare the results to the original models to demonstrate the effectiveness of our approach.
We retain the minimum and maximum length settings. The algorithm for pruning the length of each trajectory is detailed in \ref{app: prune}. Most of the original hyperparameter settings are retained as specified in the original papers. A complete list of modified hyperparameters is provided in \cref{app: hyper}.

\paragraph{Long-term Instruction Following Agents}
We use hierarchical agents to model long trajectories, which are widely used by \citep{palm-e,omnijarvis,deps}.
We utilize in-context learning, as demonstrated by JARVIS-1~\citep{jarvis-1}, and behavioral cloning, as shown by OmniJARVIS~\citep{omnijarvis}, to evaluate the long-horizon instruction-following capabilities of hierarchical agents.
JARVIS-1 is a hierarchical agent that uses the text-conditioned STEVE-1 as its policy and pretrained vision-language models as planners. We replace the text-conditioned policy with our re-trained STEVE-1 based on SBD. OmniJARVIS is a vision-language action agent built on FSQ-GROOT, which encodes instructions into discrete tokens rather than continuous embeddings. It is trained on behavior trajectories encoded into unified token sequences.

\section{Experiments and Analysis}
\label{sec:results}

\begin{table*}[!ht]
\centering
\renewcommand{\arraystretch}{1.2}
\resizebox{0.99\linewidth}{!}{
\begin{tabular}{@{}cccccccccccccc@{}}
\toprule
\textbf{Task} & \makecell[c]{\includegraphics[height=0.6cm,valign=c,keepaspectratio]{figures/icons/furnace.png} \\ use \\ furnace}
& \makecell[c]{\includegraphics[height=0.6cm,valign=c,keepaspectratio]{figures/icons/hunt_sheep.png} \\ hunt \\ sheep}
& \makecell[c]{\includegraphics[height=0.6cm,valign=c,keepaspectratio]{figures/icons/sleep.png} \\ sleep \\ in bed}
& \makecell[c]{\includegraphics[height=0.6cm,valign=c,keepaspectratio]{figures/icons/torch.png} \\ use \\ torch}
& \makecell[c]{\includegraphics[height=0.6cm,valign=c,keepaspectratio]{figures/icons/boat.png} \\ use \\ boat}
& \makecell[c]{\includegraphics[height=0.6cm,valign=c,keepaspectratio]{figures/icons/bow.png} \\ use \\ bow}
& \makecell[c]{\includegraphics[height=0.6cm,valign=c,keepaspectratio]{figures/icons/stone.png} \\ collect \\ stone}
& \makecell[c]{\includegraphics[height=0.6cm,valign=c,keepaspectratio]{figures/icons/seagrass.png} \\ collect \\ seagrass}
& \makecell[c]{\includegraphics[height=0.6cm,valign=c,keepaspectratio]{figures/icons/wood.png} \\ collect \\ wood}
& \makecell[c]{\includegraphics[height=0.6cm,valign=c,keepaspectratio]{figures/icons/tree.png} \\ find and \\ collect wood}
& \makecell[c]{\includegraphics[height=0.6cm,valign=c,keepaspectratio]{figures/icons/dirt.png} \\ collect \\ dirt}
& \makecell[c]{\includegraphics[height=0.6cm,valign=c,keepaspectratio]{figures/icons/grass.png} \\ collect \\ grass}
& \textbf{Average}\\
\midrule
\text{-} & 21.0\% & 26.0\% & 100.0\% & \textbf{97.0\%} & 71.0\% & 30.0\% & 21.7 & 34.0\% & 76.0\% & 19.0\% & 14.5 & 9.5 & 63.3 \\
\textit{good} dist. & 38.0\% & \textbf{65.0\%} & 90.0\% & 89.0\% & 80.0\% & 10.0\% & 24.8 & 43.0\% & 96.0\% & 41.0\% & 15.5 & 19.0 & 79.9 \\
\midrule
\text{Info} & 33.0\% & 52.0\% & 100.0\% & 88.0\% & \textbf{97.0\%} & 32.0\% & 23.7 & \textbf{65.0\%} & 92.0\% & 47.0\% & 15.5 & 21.6 & 85.8 \\
\midrule
Loss (GAP=17.5) & \textbf{44.0\%} & 49.0\% & 100.0\% & 89.0\% & 91.0\% & 32.0\% & \textbf{27.3} & \textbf{65.0\%} & \textbf{96.0\%} & 38.0\% & 17.8 & 22.7 & 88.4 \\
Loss (GAP=18) & \textbf{44.0\%} & 54.0\% & 100.0\% & 94.0\% & 72.0\% & 46.0\% & 25.8 & 63.0\% & 95.0\% & \textbf{48.0\%} & 17.9 & 22.7 & 90.2 \\
Loss (GAP=18.5) & 41.0\% & 51.0\% & 100.0\% & 91.0\% & 73.0\% & 44.0\% & 26.8 & 57.0\% & 91.0\% & 45.0\% & 18.1 & 25.3 & 88.5 \\
\midrule
Both (GAP=18) & 30.0\% & 54.0\% & 100.0\% & 88.0\% & 93.0\% & \textbf{80.0\%} & 26.8 & 51.0\% & 90.0\% & 44.0\% & \textbf{19.7} & \textbf{25.4} & \textbf{91.7} \\
\bottomrule
\end{tabular}
}
\caption{\textbf{Ablation studies.} 
We report the evaluation results of GROOT trained on datasets segmented by random splitting with fixed length (-), random splitting following the length distribution of SBD (\textit{good} dist.), and SBD with different components and GAP values. A value with \% indicates the average success rate, while a value without \% indicates the average rewards. To enable a fair comparison between success rate and rewards, the average score is computed by scaling the best score to 100 and all other scores proportionally for each task. 
}
\label{tab:ablation result}
\end{table*}

In our experiments, we answer the following questions:\\
$\bullet$ Are the short segments obtained through SBD more consistent semantically, and can a better goal-conditioned policy be trained on these segments? \\
$\bullet$ Do improvements in goal-conditioned policies enhance the ability of hierarchical agents? \\
$\bullet$ Which component of SBD is more important?

\subsection{Experimental Setups}
\label{sec:setup}

\paragraph{Training Dataset}
We use OpenAI's contractor dataset 7.x (early game)~\citep{vpt} as our training dataset. It contains Minecraft offline trajectories with 68M frames with a duration of approximately 1000 hours, with at least half of the data from the first 30 minutes of the game. Our method generates a segmented dataset of 130k sub-trajectories using bc-early-game-3x~\citep{vpt} as the pretrained unconditional model. 

\paragraph{Evaluation Benchmarks}
For the goal-conditioned policies, we select basic skills such as \texttt{collect wood} and advanced skills like \texttt{use furnace} in Minecraft as the evaluation benchmark.
We test 12 different atomic skills designed based on MCU~\citep{mcu}.
All tasks are tested over 100 times, except for \texttt{sleep in bed} and \texttt{use bow}, which are evaluated 10 times using human rating because they cannot be automatically verified with in-game information.
The evaluation metrics are success rates or rewards (for finer-grained assessment of easier tasks). 
See more details in \cref{app: benchmark}.

The hierarchical agents are evaluated on long-horizon programmatic tasks that require the agents to start from an empty inventory in a new world until obtaining the final required items, such as \texttt{obtain an iron pickaxe from scratch}, which is usually a chain of atomic tasks. 
We split the long-horizon tasks into different groups, including wooden items, food, stone items, iron \etc.
For each group, we choose the tasks from JARVIS-1 benchmarks~\citep{jarvis-1} and evaluate each task over 30 times. 

\subsection{Main Results}

In this section, we first present a human analysis of the quality of the SBD-segmented videos. Then we present the results of SBD applied to the two policies and two agents in \cref{sec:details} on the benchmarks in \cref{sec:setup}.

\paragraph{Human Analysis of Segmented Videos}
We ask 10 people to analyze the SBD-segmented videos, each with 50 randomly sampled clips: On average, 54\% exhibit clear independent semantics aligned with human preference, 34\% are considered acceptable, and only 12\% could be better segmented. This further supports our method’s effectiveness.

\label{subsec:main result}
\paragraph{Short-horizon Atomic Tasks} 
As shown in \cref{main_result_short}, the policies showed substantial improvements across most tasks, with an average relative performance enhancement of 63.7\% for GROOT and 52.1\% for STEVE-1.

\paragraph{Long-horizon Programmatic Tasks}
As shown in \cref{main_result_long}, the agents have substantial improvements across most of the tasks, achieving an average relative performance enhancement of {\omnijarvisresult} for Omnijarvis and 20.8\% for JARVIS-1.

\subsection{Ablation Studies}
\label{subsec:ablation}


\begin{figure*}[ht]
\captionsetup[subfigure]{justification=centering}
    \begin{subfigure}{0.33\linewidth}
    \caption{~~~Info only.}
    \label{subfig:info}
    \includegraphics[width=0.99\linewidth]{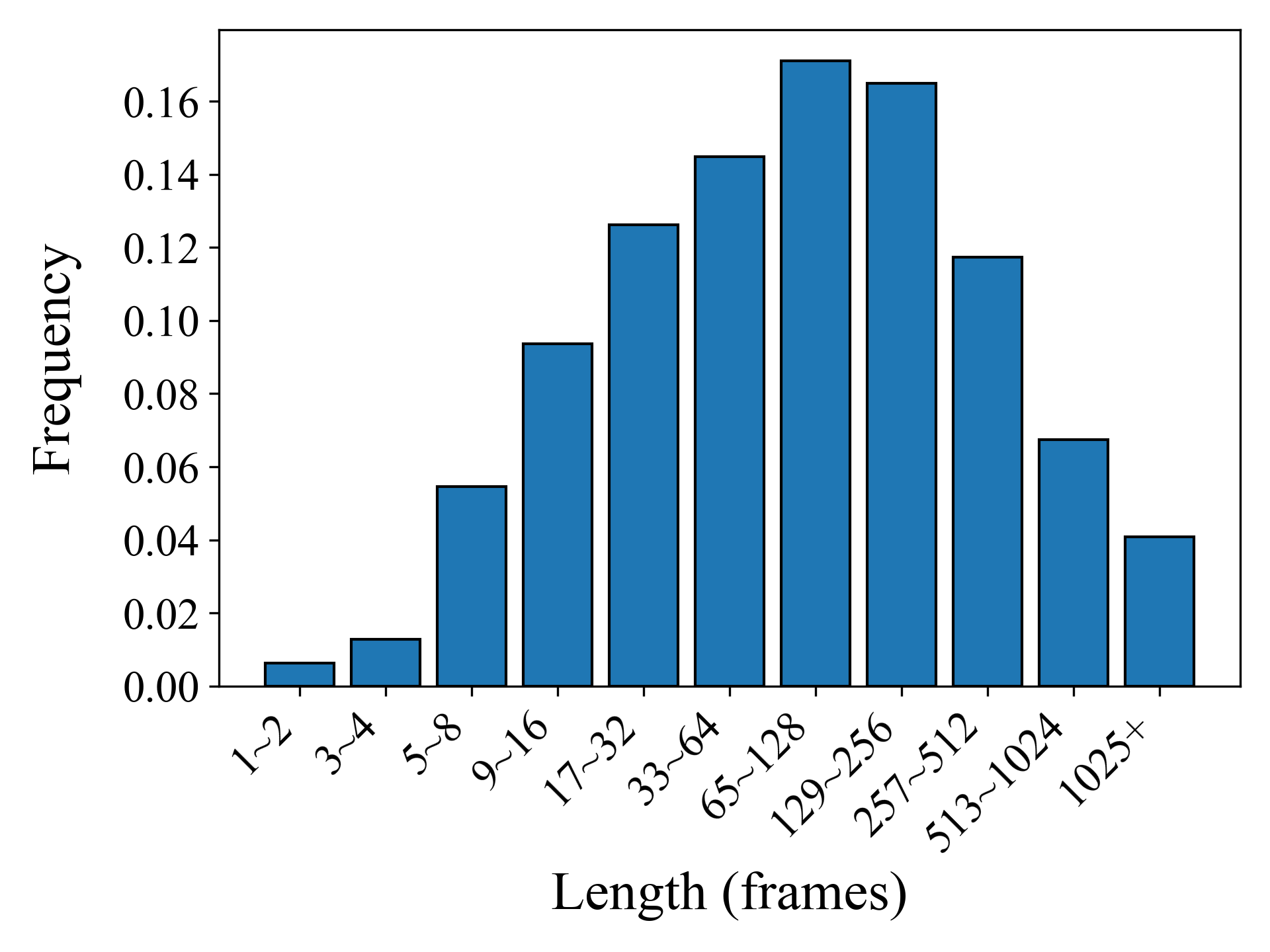}
    \end{subfigure}
    \begin{subfigure}{0.33\linewidth}
    \caption{~~~Loss only.}
    \label{subfig:loss}
    \includegraphics[width=0.99\linewidth]{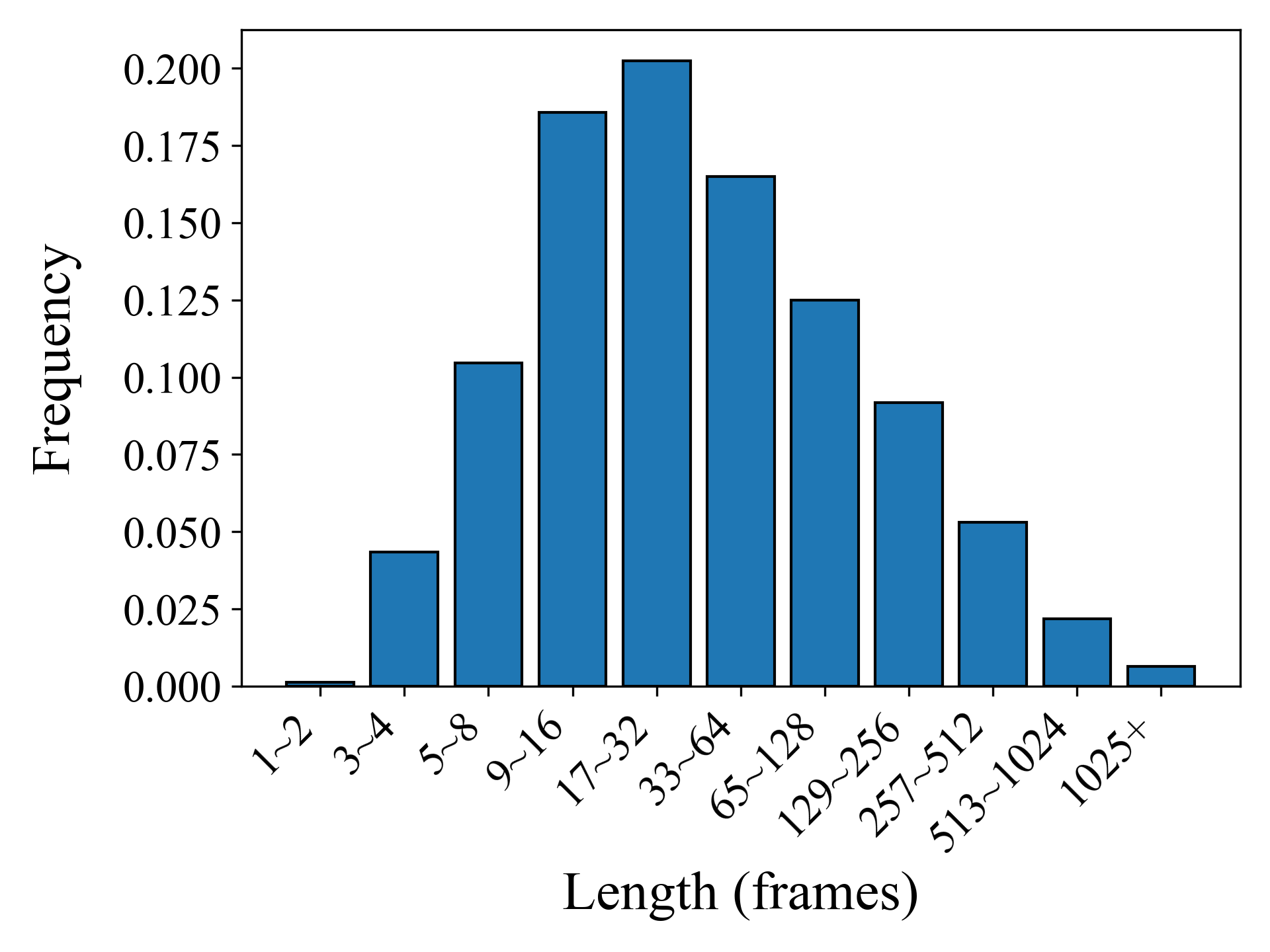}
    \end{subfigure}
    \begin{subfigure}{0.33\linewidth}
    \caption{~~~Both.}
    \label{subfig:both}
    \includegraphics[width=0.99\linewidth]{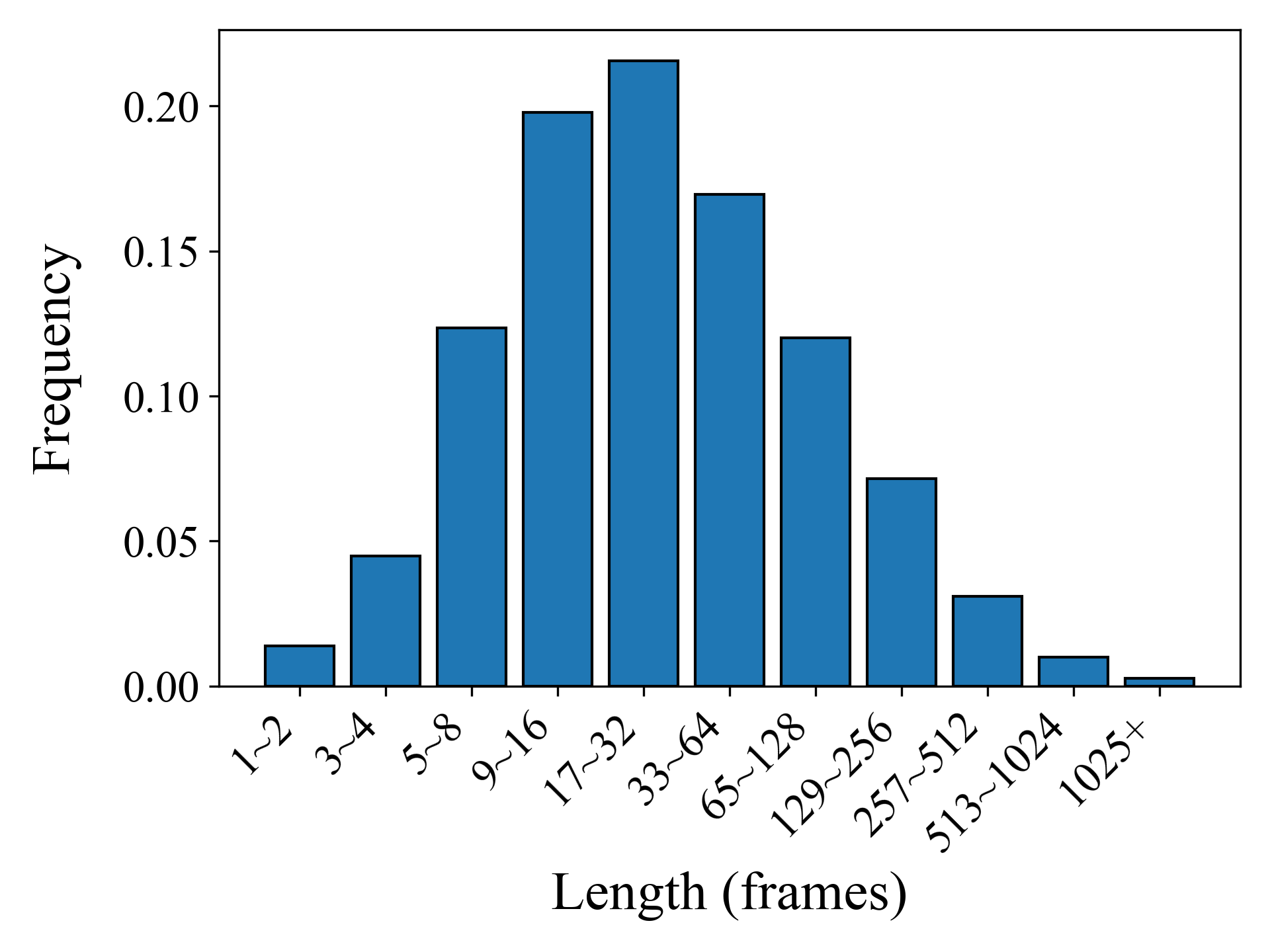}
    \end{subfigure}
    
    \caption{\textbf{The length distribution of segments, split by info and loss.} The info-only method is intrinsically semantically meaningful, suggesting that the loss-only method also identifies a semantically meaningful segmentation pattern. Furthermore, the similarity between the combined method and the loss-only method indicates that predictive loss is the primary factor in learning the segmentation pattern.
    }
    \label{fig:length}
\end{figure*}

\begin{figure*}[ht]
\begin{center}
\begin{subfigure}{0.99\linewidth}
    \includegraphics[width=0.99\linewidth]{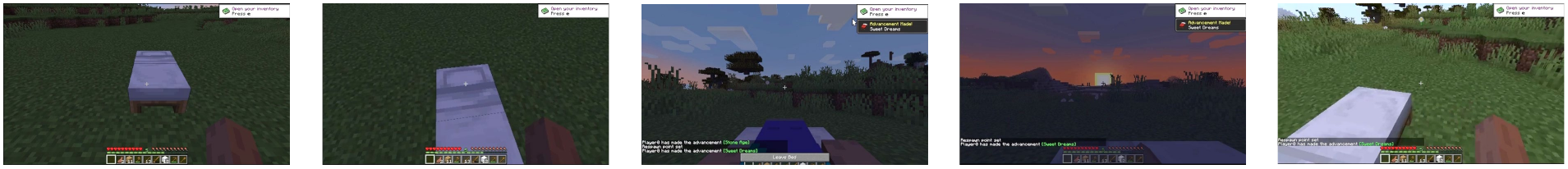}
\end{subfigure}
\begin{subfigure}{0.99\linewidth}
\includegraphics[width=0.99\linewidth]{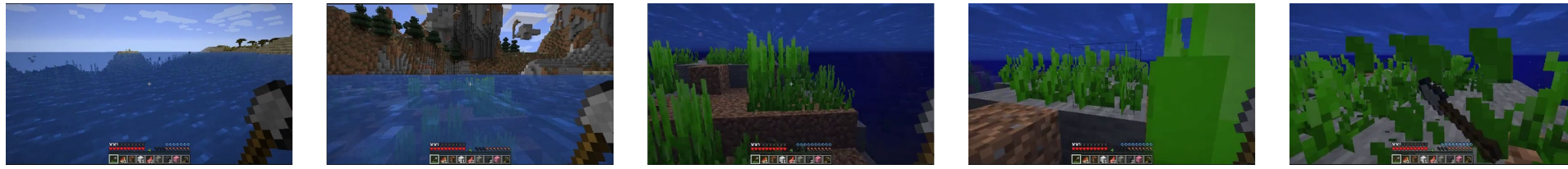}
\end{subfigure}
\end{center}
\caption{\textbf{Video Segment Examples. Top: \texttt{sleep in bed}. Bottom: \texttt{collect grass}}. Each segment is accompanied by five screenshots. The first and last screenshots represent the initial and final frames of the segment, respectively. The remaining three screenshots are manually selected to best illustrate the skill's progression. More segments can be found in \cref{app: skill videos}.
}
\label{fig:video examples}
\end{figure*}

\paragraph{SBD Components} As introduced in \cref{subsec:loss,subsec:info}, the environment information and the skill boundary detection loss are important components for accurate segmentation of long videos.
In this section, we explore the effectiveness of each component. 
We compare the full SBD with random splitting (w/o loss and info), SBD w/ info only, and SBD w/ loss only.
The results are presented in \cref{tab:ablation result}. Our findings are as follows:
(i) Using external information alone significantly improves performance, with a relative average increase of 35.5\% over the original agent.
(ii) Using loss alone results in a better performance (a relative average increase of 5.1\%) compared to using external information alone, highlighting the effectiveness of the core component of our method. This demonstrates that our method achieves strong performance even without external information.
(iii) Combining loss and external information yields the best overall performance. This suggests that each component captures unique segmentation patterns that the other misses. 

We observe performance declines when employing both components in tasks \texttt{use furnace} and \texttt{collect seagrass}. 
This may be due to both components splitting the trajectory at the same intervals but not identical steps, leading to redundant trajectory fragments.

\paragraph{Hyperparameter sensitivity} GAP is an important hyperparameter of SBD. To verify the stability of our method, we also conduct an ablation study on different values of GAP. As shown in \cref{tab:ablation result}, our method is not sensitive to it (only 2.0\% relative average performance decline).

\paragraph{Improved Baseline} A simple idea to improve random splitting is to follow a similar length distribution as SBD. We add this improved baseline to show that our method also catches features other than length distribution. As shown in \cref{tab:ablation result}, while a better length distribution can improve baseline performance (by 26.2\% relatively), our method still outperforms it (by 14.8\% relatively).


\subsection{Visualizations}
\label{subsec:visual}
\paragraph{Length Distribution}
\cref{fig:length} shows the length distribution of sub-trajectories split by different methods with or without the two components of SBD. We observe that both loss-only and info-only methods follow similar distributions, where the log length and frequency of sub-trajectories exhibit a normal distribution. Since the info-only method is intrinsically semantically meaningful, this provides indirect evidence that the loss-only method also identifies a semantically meaningful segmentation pattern as the info-only method. Furthermore, when combining loss with external information, the distribution does not significantly differ from the loss-only method. This observation highlights that the predictive loss is the key driver in learning the segmentation pattern, with the external information serving a more supplementary role.

\paragraph{Skill Videos}
We sample one long unsegmented video segmented by our method for each skill in the datasets, such as \texttt{sleep in bed} and \texttt{collect seagrass} (see \cref{fig:video examples}). More segments can be seen in \cref{app: skill videos}.

\section{Related Works}
\label{sec:related}
\paragraph{Learning from Unsegmented Demonstrations} 
Prior work on learning from unsegmented demonstrations has proposed various methods to segment trajectories into sub-trajectories. For instance, some methods~\citep{temporal_vae, skid_raw, compile} utilize a variational autoencoder~\citep{vae} model to generate pairs of skill types and boundaries, $(g_i, t_i)_{i=1}^k$, from a given trajectory. TACO~\citep{taco} adopts a weakly supervised framework that identifies skill boundaries, $(t_1,\ldots t_k)$, based on skill sketches, $(g_1,\ldots g_k)$, by solving a sequence alignment problem. BUDS~\citep{buds} constructs hierarchical task structures of demonstration sequences using a bottom-up strategy, deriving temporal segments through agglomerative clustering of the actions. FAST~\citep{FAST} proposes an approach based on the byte-pair encoding~\citep{bpe} algorithm for segmentation and tokenization of robot action trajectories.

\paragraph{Option Learning}
Option learning~\citep{option} is a framework to learn temporal abstractions (\ie, options or skills) directly from the environment's reward signal, primarily through online reinforcement learning. Recent approaches~\citep{option-critic, option-discovery, flexible-option} rely on policy gradient methods to learn the option policies. DIAYN~\citep{DIAYN} proposes a method to generate a reward function adaptively in the absence of reward signals, resulting in the unsupervised emergence of diverse skills.

\paragraph{Event Segmentation Theory}
In neuroscience, episodic memory is the memory of everyday events consisting of short slices of experience~\citep{episodic_memory}. Event Segmentation Theory (EST)~\citep{EST} offers a theoretical perspective on how the neurocognitive system splits a long flow of memory into short events. According to EST, observers build event models of the current situation to generate predictions of future perceptual input. When errors in predictions arise, an event boundary is perceived, and the event model is reset and rebuilt. Our method employs a pretrained model to predict the agent's action based on past observations and uses the predictive loss as an indicator to detect skill boundaries.

\paragraph{Agents in Minecraft} 
Many agents have been developed to interact with Minecraft environments~\citep{zhao2024optimizing,jiang2025reinforcement}. For short-horizon tasks, methods typically employ behavioral cloning or reinforcement learning, as seen in works like VPT~\citep{vpt}, which annotates a large Minecraft YouTube video dataset with actions and trains the first foundational agent in the domain, and its derivatives~\citep{steve1,groot1,yuan2024pre,cai2023open,li2025jarvis}. For long-horizon, programmatic tasks, large language models (LLMs) are used as planners combined with skill policies~\citep{deps,minedreamer,plan4mc,mp5,jarvis-1,li2024optimus,ui-tars-15-seed}, and some methods~\citep{jarvis-1,gitm,voyager} employ explicit memory mechanisms to enhance the long-horizon capabilities of LLM agents. Additionally, recent advances~\citep{omnijarvis,steve2} have explored using end-to-end vision-language models (VLMs) to directly follow human instructions. Finally, some research~\citep{luban,zhang2023creative} focuses on open-ended creative tasks such as building and decoration.

\section{Conclusion}
\label{sec:conclusion}

In this paper, we propose a novel temporal video segmentation algorithm to address the challenges of open-world skill discovery from unsegmented demonstration videos. To generate a segmented dataset of video clips with independent skills, our algorithm detects the potential skill boundaries based on the predictive loss of a pretrained action-prediction model. 
It does not require any manual annotations and can be employed directly on massive internet gameplay videos. 

\section*{Acknowledgement}

This work is funded in part by the National Science and Technology Major Project 2022ZD0114902.
We thank a grant from CCF-Baidu Open Fund.
{
    \small
    \bibliographystyle{ieeenat_fullname}
    \bibliography{main}
}

\clearpage
\maketitlesupplementary
\appendix
\section{Discussion and Future work}
\label{sec:discussion}

\paragraph{Computational Cost} 
Our method takes approximately 1 minute to process a 5-minute video on an NVIDIA RTX 3090Ti, due to the need to predict actions across the entire trajectory. Since data preprocessing is a one-time cost, spending only 1/20 of the video duration with 4 GPUs is acceptable. Also, it is more efficient than those methods relying on human or VLM annotations while ensuring quality.

\paragraph{Failure Case Due to Occasional Assumption Violation}
Our method is occasionally unstable in action-intensive scenes like combat, where performance drops on task \texttt{combat spiders} (29\% to 20\%). See \cref{app: skill videos} for visualization of this failure case. This is likely due to a violation of \cref{assumption: confidence} (Skill Confidence), as indicated by sharp prediction loss fluctuations. However, such issues are not observed elsewhere. Since overly short segments (<~5 frames) comprise less than 10\% of all videos, overall effectiveness remains unaffected.

\paragraph{Choice of Loss Term}
Currently our method detects changes in $P(a_t|o_{1:t})$. A possible alternative is to use $P(o_{t+1}|o_{1:t})$. We don't use it because it does not directly reflect skill changes as $P(a_t|o_{1:t})$ does, and current models~\citep{minedreamer} for predicting future observations suffer from hallucinations. We leave the attempt of this alternative to future work.

\paragraph{Scalability and Broader Applicability}
Since our method is label-free, it has the potential to segment and train on the numerous gameplay trajectories on YouTube. We also anticipate the potential application of our method to other larger datasets and other open-world video games beyond Minecraft.

\section{Proofs}
In this section, we provide a detailed proof of \cref{thm:non-switching and switching}, regarding the bounds of relative predictive probability under scenarios of skill transition and non-transition. We prove the lower bound and upper bound respectively in the following two subsections. 
\label{app: proof}
\subsection{Proof of Upper Bound Under Skill Non-Transition}
\begin{align*}
P(\frac{P(a_{t+1}|o_{1:t+1})}{(\prod_{i=1}^t P(a_{i}|o_{1:i}))^{1/t}}>\frac{(K-1)c}{K}) 
&> P(P(a_{t+1}|o_{1:t+1})>\frac{(K-1)c}{K})
\\&> P(P(\pi_{t+1}=\pi_{t}|o_{1:t+1})\pi_t(a_{t+1}|o_{1:t+1})>\frac{(K-1)c}{K}) && \text{(\cref{eq:transition}})
\\&> P(\pi_t(a_{t+1}|o_{1:t+1})>c) && \text{(\cref{assumption: consistency})}
\\&= P(\pi_{t+1}(a_{t+1}|o_{1:t+1})>c)
\\&> 1-\delta   && \text{(\cref{assumption: confidence})}
\end{align*}

\subsection{Proof of Lower Bound Under Skill Transition}
We first magnify the likelihood ratio between the next action and the average action history,
\begin{align*}
\frac{P(a_{t+1}|o_{1:t+1})}{(\prod_{i=1}^t P(a_{i}|o_{1:i}))^{1/t}}
&< \frac{P(\pi_{t+1}=\pi_t|o_{1:t+1})\pi_t(a_{t+1}|o_{1:t+1}) + P(\pi_{t+1}\ne\pi|o_{1:t+1})}{(\prod_{i=1}^t P(a_{i}|o_{1:i}))^{1/t}} && \text{(\cref{eq:transition})}
\\&< \frac{P(\pi_{t+1}=\pi_t|o_{1:t+1})\pi_t(a_{t+1}|o_{1:t+1}) + P(\pi_{t+1}\ne\pi|o_{1:i})}{\frac{K-1}{K}(\prod_{i=1}^t\pi_{i-1}(a_{i}|o_{1:i}))^{1/t}} && \text{(\cref{eq:transition} and} \\ & && \text{ \cref{assumption: consistency})}
\\&< \frac{\pi_t(a_{t+1}|o_{1:t+1})+\frac{1}{K}}{\frac{K-1}{K}(\prod_{i=1}^t\pi_{i-1}(a_{i}|o_{1:i}))^{1/t}}    && \text{(\cref{assumption: consistency})}
\\&= \frac{K}{K-1}\frac{\pi_t(a_{t+1}|o_{1:t+1})}{\prod_{i=1}^t\pi_{t}(a_{i}|o_{1:i}))^{1/t}} + \frac{1}{(K-1)\prod_{i=1}^t\pi_{i}(a_{i}|o_{1:i}))^{1/t}}
\\&< \frac{Km}{2(K-1)} + \frac{1}{(K-1)\prod_{i=1}^t\pi_{i}(a_{i}|o_{1:i}))^{1/t}} && \text{(\cref{assumption: deviance})}
\end{align*}
Therefore,
\begin{align*}
 &P(\frac{P(a_{t+1}|o_{1:t+1})}{(\prod_{i=1}^t P(a_{i}|o_{1:i}))^{1/t}} < \frac{Km}{2(K-1)} + \frac{1}{c(K-1)})
\\&> P(\frac{Km}{2(K-1)} + \frac{1}{(K-1)\prod_{i=1}^t\pi_{i}(a_{i}|o_{1:i}))^{1/t}}<\frac{Km}{2(K-1)} + \frac{1}{c(K-1)})
\\&= P(\prod_{i=1}^t\pi_{i}(a_{i}|o_{1:i}))^{1/t} > c)
\\&> \prod_{i=1}^{t}P(\pi_{i}(a_{i}|o_{1:i}) > c)
\\&> (1-\delta)^t > 1-t\delta   && \text{(\cref{assumption: confidence})}
\end{align*}

\section{Minecraft Environment}
\label{app:minecraft_env}
Minecraft is an extremely popular sandbox game that allows players to freely create and explore their
world. This game has infinite freedom, allowing players to change the world and ecosystems through
building, mining, planting, combating, and other methods. It is precisely because
of this freedom that Minecraft becomes an excellent AI testing benchmark. In this
game, AI agents need to face situations that are highly similar to the real world, making judgments
and decisions to deal with various environments and problems. By using Minecraft, AI researchers can more
conveniently simulate various complex and diverse environments and tasks, thereby improving the
practical value and application of AI technology.

Our Minecraft environment is a hybrid between MineRL~\citep{minerl} and the MCP-Reborn (\href{https://github.com/Hexeption/MCP-Reborn}{https://github.com/Hexeption/MCP-Reborn}) Minecraft modding package. Unlike the regular Minecraft game, in which the server (or the "world") always runs at 20Hz while the client’s rendering speed can typically reach 60-100Hz, the frame rate is fixed at 20 fps for the client in our experiments. The action and
observation spaces in our environment are identical to what a human player can operate and observe on their device when
playing the game. These details will be further explained in subsequent subsections.

\subsection{Minecraft Game World Setting}
We choose Minecraft version 1.16.5’s survival mode as our experiment platform. In this mode, the
agent may encounter situations that result in its death, such as being burned by lava or a campfire, getting killed by hostile
mobs, or falling from great heights. When this happens, the agent will lose all its items and respawn at a random location
near its initial spawn point within the same Minecraft world or at the last spot it attempted to sleep. Importantly, even after
dying, the agent retains knowledge of its previous deaths and can adjust its actions accordingly since there is no masking of
policy state upon respawn.

\subsection{Observation Space}
The environmental observation space consists of two parts. The first part is the raw pixels from the
Minecraft game that players would see, including overlays such as the hotbar, health indicators, and
animations of a moving hand in response to attack or "use" actions. The rendering resolution of Minecraft is 640x360; however, in our experiments, we resize images to 128x128 for better computational efficiency while maintaining discernibility.
The second part includes auxiliary information about the agent’s current environment, such as its location and weather conditions.
Human players can obtain this information by pressing F3. The specific observation details we include
are shown in \cref{tab:observation_space}. 

\begin{table}[H]
\centering
\resizebox{0.8\linewidth}{!}{
\renewcommand\arraystretch{1.1}
\begin{tabular}{@{}lll@{}}
\toprule
Sources         & Shape       & Description                                                                                                                                                                                                                                                                                                                                \\ \midrule
pov (raw pixel)             & \makecell[c]{(640, 360, 3) \\ to (128,128,3)} & Ego-centric RGB frames.                                                                                                                                                                                                                                                                                                                    \\ \midrule
player\_pos     & (5,)        & The coordinates of (x,y,z), pitch, and yaw of the agent.                                                                                                                                                                                                                                                                                   \\ \midrule
location\_stats & (9,)        & \begin{tabular}[c]{@{}l@{}}The environmental information of the agent's current position, \\ including \texttt{biome\_id}, \texttt{sea\_level}, \texttt{can\_see\_sky}, \texttt{is\_raining} etc.\end{tabular}                                                                                                                                                                 \\ \midrule
inventory       & (36,)       & \begin{tabular}[c]{@{}l@{}}The items in the current inventory of the agent, including \\ the \texttt{type} and corresponding \texttt{quantity} of each item in each slot. \\ If there is no item, it will be displayed as \texttt{air}.\end{tabular}                                                                                                                \\ \midrule
equipped\_items & (6,)        & \begin{tabular}[c]{@{}l@{}}The current equipment of the agent, including \texttt{mainhand}, \texttt{offhand}, \\ \texttt{chest}, \texttt{feet}, \texttt{head}, and \texttt{legs} slots. Each slot contains \texttt{type}, \texttt{damage}, \\ and \texttt{max\_damage} information.\end{tabular}                                                                                                                           \\ \midrule
event\_info     & (5,)        & \begin{tabular}[c]{@{}l@{}}The events that occur in the current step of the game, including \\ \texttt{pick\_up} (picking up items), \texttt{break\_item} (breaking items), \\ \texttt{craft\_item} (crafting items using a crafting table or crafting grid), \\ \texttt{mine\_block} (mining blocks by suitable tools), and \\ \texttt{kill\_entity} (killing game mobs).\end{tabular} \\ \bottomrule
\end{tabular}}
\captionsetup{justification=centering}
\caption{The observation space we use in Minecraft.}
\label{tab:observation_space}
\end{table}

During the actual inference process, the controllers (VPT~\citep{vpt}, GROOT~\citep{groot1}, STEVE-1~\citep{steve1}) only perceive
the raw pixels. The agents (JARVIS-1~\citep{jarvis-1}, OmniJarvis~\citep{omnijarvis}) can access auxiliary information from the environment to generate the text condition of the controller.

\subsection{Action Space}
Our action space includes almost all actions directly available to human players, such as keypresses,
mouse movements, and clicks. It consists of two parts: the mouse and the keyboard. When in-game GUIs are not open, the mouse movement is responsible for changing the player’s camera perspective. When a GUI is open, it moves the cursor. The left and
right buttons are responsible for attacking and using items. The keyboard is mainly responsible for
controlling the agent’s movement. We use the same joint hierarchical action space as VPT~\citep{vpt}, which combines button space and camera space. Button space encodes all combinations of possible keyboard operations (excluding mutually exclusive actions such as \texttt{forward} and \texttt{back}) and a
flag indicating whether the mouse is used, resulting in a total of 8461 candidate actions. The camera
space discretizes the range of one mouse movement into 121 actions. Therefore, the action head of
the agent is a multi-classification network with 8461 dimensions and a multi-classification network
with 121 dimensions. We also filter out frames (both the observation and action) with null action as VPT~\citep{vpt}.

In addition, we abstract the crafting and smelting actions with GUI into functional binary actions,
which are the same as MineDojo (Fan et al., 2022). These advanced actions are only used by the agents (JARVIS-1~\citep{jarvis-1}, OmniJarvis~\citep{omnijarvis}). The detailed action space is described in \cref{tab:action_space}.

\begin{table}[H]
\centering
\resizebox{0.8\linewidth}{!}{
\renewcommand\arraystretch{1.1}
\begin{tabular}{@{}cccl@{}}
\toprule
\textbf{Index} & \textbf{Action}  & \textbf{Human Action} & \textbf{Description}                                                                                                                                                                                                                                                                                             \\ \midrule
1              & Forward          & key W                 & Move forward.                                                                                                                                                                                                                                                                                                    \\
2              & Back             & key S                 & Move backward.                                                                                                                                                                                                                                                                                                   \\
3              & Left             & key A                 & Strafe left.                                                                                                                                                                                                                                                                                                     \\
4              & Right            & key D                 & Strafe right.                                                                                                                                                                                                                                                                                                    \\
5              & Jump             & key Space             & Jump. When swimming, keeps the player afloat.                                                                                                                                                                                                                                                                    \\
6              & Sneak            & key left Shift        & Slowly move in the current direction of movement. \\
7              & Sprint           & key left Ctrl         & Move quickly in the direction of the current motion.                                                                                                                                                                                                                                                                 \\
8              & Attack           & left Button     & Destroy blocks (hold down); Attack entity (click once).                                                                                                                                                                                                                                                          \\
9              & Use              & right Button    & Interact with the block that the player is currently looking at.                                                                                                 \\
10             & hotbar.{[}1-9{]} & keys 1 - 9            & Selects the appropriate hotbar item.                                                                                                                                                \\
11             & Yaw              & move Mouse X      & Turning; aiming; camera movement.Ranging from -180 to +180.                                                                                                                                                                                                                                                      \\
12             & Pitch            & move Mouse Y      & Turning; aiming; camera movement.Ranging from -180 to +180.  \\

\midrule

13             & Equip            & -      & Equip the item in the main hand from the inventory.  \\

14             & Craft            & -      & Execute a crafting recipe to obtain a new item.  \\

15             & Smelt            & -      & Execute a smelting recipe to obtain a new item.

\\ \bottomrule
\end{tabular}}
\captionsetup{justification=centering}
\caption{The action space we use in Minecraft.}
\label{tab:action_space}
\end{table}

\section{Experiment Details}
In this section, we provide our experiment details, including benchmark, training details for GROOT and STEVE-1, how we use event-based information as external auxiliary information, and the length pruning algorithm for sub-trajectories.
\subsection{Minecraft Skill Benchmark}
\label{app: benchmark}
MCU~\citep{mcu} is a diverse benchmark that can comprehensively evaluate the mastery of atomic skills by agents in Minecraft. Since our method is an improvement on dataset processing instead of a new model architecture, it cannot help agents learn new skills that they are completely incapable of acquiring previously. Therefore, we choose 10 early game skills from the original benchmark that the original agent can already grasp at a basic level. Besides, we add \texttt{use torch}, which is also a useful skill in Minecraft. We also add \texttt{find and collect wood}, an enhanced version of the skill \texttt{collect wood}, which requires the agent to start from the plains instead of the forest, aiming to test its ability to explore and find trees. 

Details of the 12 skills in our early game benchmark are shown in \cref{table: benchmark}. For each skill, we include the evaluation metric and a brief description of what the skill is. For skills "Sleep" and "Use bow", GROOT performs very well on the original metric, so we design a manual metric that better assesses its actual performance. STEVE-1 can take text as prompts, which are also listed in the table. For some of the skills, it is tricky to find proper text prompts, so we use visual prompts instead.  
\begin{table}[!ht]
\renewcommand{\arraystretch}{1.13}
\centering
\begin{tabularx}{0.99\linewidth}{@{}>{\centering\arraybackslash}m{1in} >{\centering\arraybackslash}m{2in} >{\centering\arraybackslash}m{2in} >{\centering\arraybackslash}m{1in}@{}}
\toprule
\textbf{Skill} & \textbf{Metric} & \textbf{Description}  & \textbf{Text Prompt for STEVE-1} \\ \midrule
Use furnace  & Craft item cooked mutton  & Given a furnace and some mutton and coal, craft a cooked mutton.  & - \\ \hline
Hunt Sheep & Kill entity sheep  & Summon some sheep before the agent, hunt the sheep.  & - \\ \hline
Sleep in bed & \textbf{STEVE-1}: Use item white bed.  \textbf{GROOT}: Sleep in bed properly. & Given a white bed, sleep on it.  & Sleep in bed. \\ \hline
Use torch & Use item torch & Give some torches, use them to light up an area. Time is set at night. & Use a torch to light up an area. \\ \hline
Use boat & Use item birch boat & Given a birch boat, use it to travel on water. The biome is ocean. & - \\ \hline
Use bow & \textbf{STEVE-1}: Use item bow.  \textbf{GROOT}: Use item bow 20\%, Shoot in distance 40\%, take aim 40\% & Given a bow and some arrows, shoot the sheep summoned before the agent. & - \\ \hline
Collect stone & Mine block stone (cobblestone, iron, coal, diamond) & Given an iron pickaxe, collect stone starting from cave. Night vision is enabled. & Collect stone. \\ \hline
Collect seagrass & Mine block seagrass (tall seagrass, kelp) & Given an iron pickaxe, collect seagrass starting from ocean. & -  \\ \hline
Collect wood & Mine block oak (spruce, birch, jungle, acacia) log & Given an iron pickaxe, collect wood starting from \textbf{forest}. & Chop a tree.    \\ \hline
Find and collect wood & Mine block oak (spruce, birch, jungle, acacia) log & Given an iron pickaxe, \textbf{find} and collect wood starting from \textbf{plains}. & Chop a tree. \\ \hline
Collect dirt & Mine block dirt (grass block) & Given an iron pickaxe, collect dirt starting from plains.  & Collect dirt. \\ \hline
Collect grass & Mine block grass (tall grass) & Given an iron pickaxe, collect grass starting from plains.  & Collect grass. \\
\bottomrule
\end{tabularx}
\caption{Details of 12 atomic skills in our Minecraft skill benchmark for testing GROOT and STEVE-1.}
\label{table: benchmark}
\end{table}

\subsection{Training Details}
\label{app: hyper}
The modified hyperparameters for GROOT and STEVE-1 are listed in \cref{tab:hyper}. We adjust parallel GPUs, gradient accumulation batches, and batch sizes to better align with our available computing resources. To speed up convergence without compromising performance, we double the learning rate for GROOT. The total number of frames for STEVE-1 is also modified, as we change the training dataset from a mixed subset of 8.x (house building from scratch), 9.x (house building from random materials), and 10.x (obtaining a diamond pickaxe) to the full 7.x dataset (early game), which better suits our benchmark tasks.~\footnote{OpenAI released five subsets of contractor data: 6.x, 7.x, 8.x, 9.x, and 10.x.}

All models are trained parallelly on four NVIDIA RTX 4090Ti GPUs. We follow the same policy training pipeline as the original paper, except for the modified hyperparameters mentioned above. GROOT is trained for three epochs of our dataset. STEVE-1, however, is only trained for 1.5 epochs of our dataset, because we observe that the model starts to overfit on the dataset if it is trained further.   

\begin{table}[!ht]
\parbox{.5\linewidth}{
    \centering
    \renewcommand{\arraystretch}{1.2}
    \begin{tabular}{@{}cc@{}}
    \toprule
    Hyperparameter & Value  \\ \midrule
    Learning Rate & 0.00004 \\
    Parallel GPUs & 4   \\
    Accumulate Gradient Batches  & 1    \\
    Batch Size & 8  \\
    \bottomrule
    \end{tabular}
}
\parbox{.5\linewidth}{
    \centering
    \renewcommand{\arraystretch}{1.2}
    \begin{tabular}{@{}cc@{}}
    \toprule
    Hyperparameter & Value  \\ \midrule
    Parallel GPUs & 4   \\
    Accumulate Gradient Batches  & 4    \\
    Batch Size & 4  \\
    n\_frames & 100M \\
    \bottomrule
    \end{tabular}
}
\label{tab:hyper}
\caption{Modified hyperparameters for training controllers. 
\textbf{(Left)} GROOT.
\textbf{(Right)} STEVE-1.}
\end{table}

\subsection{Event-Based Information}
\label{app: event}
Within the Minecraft environment, we focus on events in the categories "use item", "mine block", "craft item", and "kill entity", as they reflect the player's primary activities. Multiple events may occur simultaneously and are recorded as a set. Only the final step of a repeating sequence of sets of events is marked as positive. For example, in the sequence: ("use item iron pickaxe", "mine block iron ore"), ("use item iron pickaxe", "mine block iron ore"), ("use item iron pickaxe", "mine block diamond ore"), ("use item torch"), ("use item torch"), the second, third, fifth steps will be marked as positive. Additionally, for any step $t$ that includes a "kill entity" event, we mark $t+16$ as positive instead of $t$, because there will be a short death animation that follows the entity’s death, and we want it to be included in the same segment.

\subsection{Length Pruning Algorithm}
\label{app: prune}
STEVE-1 forces a minimum and maximum length of trajectories in its method, so we need a length pruning algorithm. In our implementation, the length of each trajectory is pruned as follows: If a trajectory is too short, it is merged with subsequent trajectories until it meets the minimum length requirement. If this results in a trajectory exceeding the maximum length, it is truncated, and the remainder forms the beginning of the next trajectory. For instance, given a minimum and maximum length of 15 and 200, sub-trajectories of lengths 12, 12, 6, 196, and 37 are adjusted to 24, 200, and 39. Here, the first two sub-trajectories are merged, while the 6 is combined with 196 but truncated at 200, with the remaining 2 merged into the next trajectory. There might be more sophisticated strategies, but we use this straightforward algorithm for simplicity.

\section{Examples of Skill Videos}
We sample one video segmented by our method for each skill in the benchmark and an extra video showing a failure case, presenting each video with five screenshots. The first and last screenshots correspond to the first and last frames of the video, while the other three are manually selected to best illustrate the progression of the skill.
\label{app: skill videos}
\begin{itemize}
    \item smelt food 
        \begin{center}
        \includegraphics[width=0.99\linewidth]{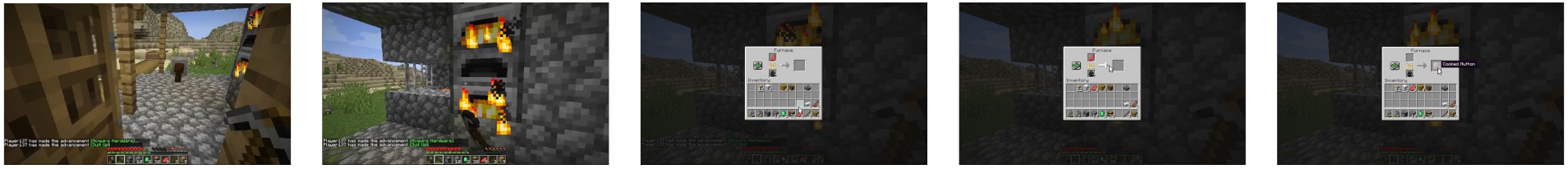}
        \end{center}
    \item hunt sheep
        \begin{center}
        \includegraphics[width=0.99\linewidth]{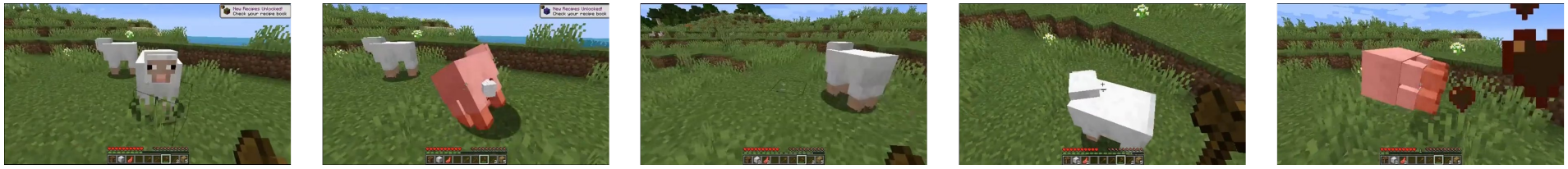}
        \end{center}
    \item sleep
        \begin{center}
        \includegraphics[width=0.99\linewidth]{figures/skills/sleep.pdf}
        \end{center}
    \item use torch
        \begin{center}
        \includegraphics[width=0.99\linewidth]{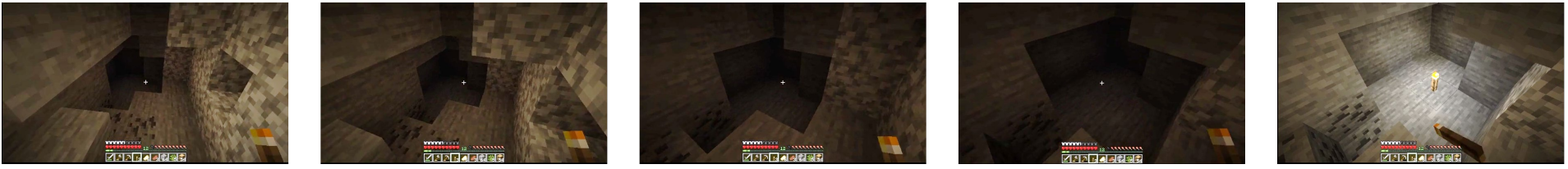}
        \end{center}
    \item use boat
        \begin{center}
        \includegraphics[width=0.99\linewidth]{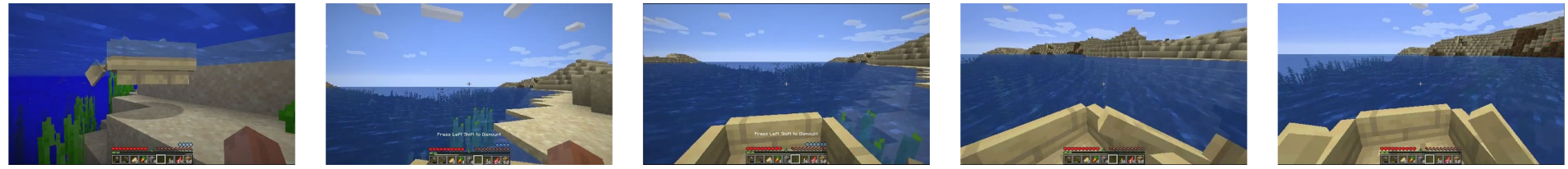}
        \end{center}
    \item use bow
        \begin{center}
        \includegraphics[width=0.99\linewidth]{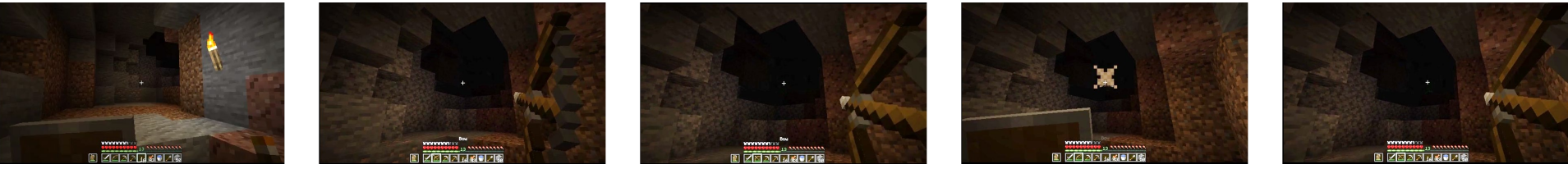}
        \end{center}
    \item collect stone
        \begin{center}
        \includegraphics[width=0.99\linewidth]{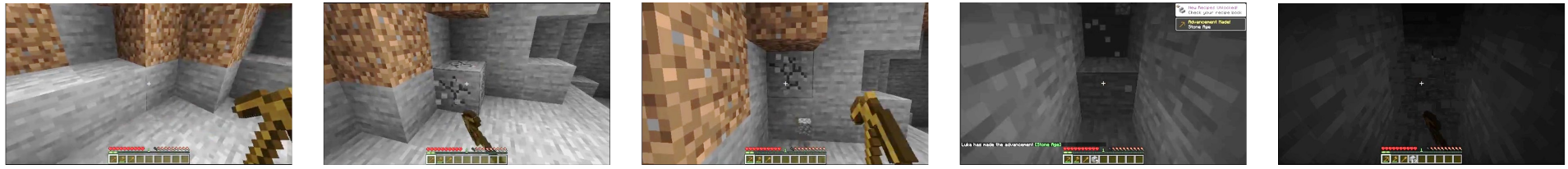}
        \end{center}
    \item collect seagrass
        \begin{center}
        \includegraphics[width=0.99\linewidth]{figures/skills/collect_seagrass.pdf}
        \end{center}
    \item collect wood
        \begin{center}
        \includegraphics[width=0.99\linewidth]{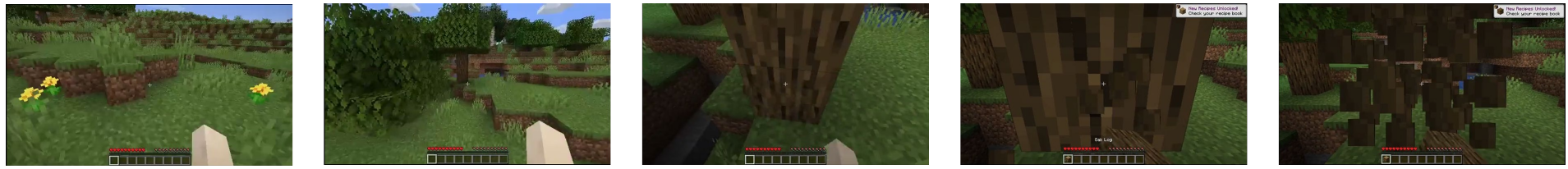}
        \end{center}
    \item collect dirt
        \begin{center}
        \includegraphics[width=0.99\linewidth]{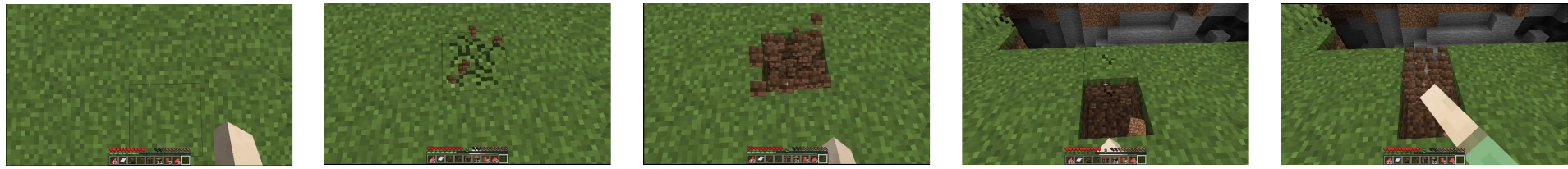}
        \end{center}
    \item collect grass
        \begin{center}
        \includegraphics[width=0.99\linewidth]{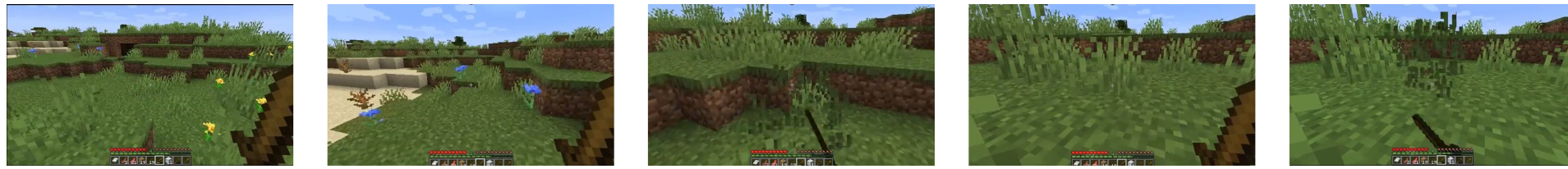}
        \end{center}
    \item combat spider (failure case: overly short)
        \begin{center}
        \includegraphics[width=0.99\linewidth]{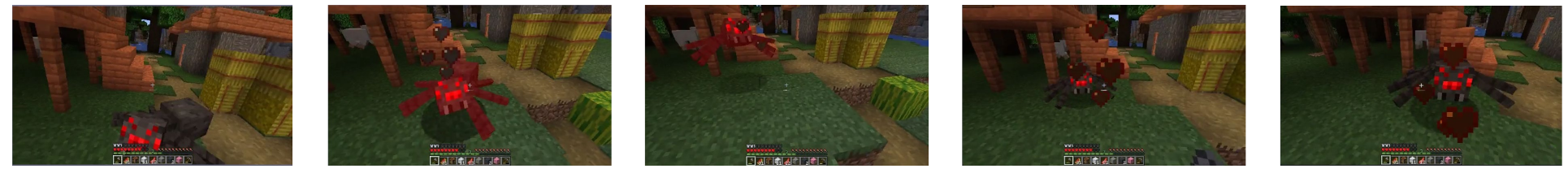}
        \end{center}
\end{itemize}

\end{document}